\documentclass{article}
\usepackage[preprint]{colm2026_conference}

\usepackage{microtype}
\usepackage{hyperref}
\usepackage{url}
\usepackage{booktabs}
\usepackage{lineno}

\usepackage{amsmath,amsfonts,bm}

\def\eqref#1{equation~\ref{#1}}

\def\1{\bm{1}}

\def\vx{{\bm{x}}}

\DeclareMathAlphabet{\mathsfit}{\encodingdefault}{\sfdefault}{m}{sl}
\SetMathAlphabet{\mathsfit}{bold}{\encodingdefault}{\sfdefault}{bx}{n}

\newcommand{\E}{\mathbb{E}}

\newcommand{\KL}{D_{\mathrm{KL}}}

\usepackage{amsmath,amssymb,amsthm,mathtools,bm}
\usepackage{thm-restate}
\makeatletter
\newcommand{\labelonce}[1]{\ifthmt@thisistheone\label{#1}\fi}
\makeatother
\usepackage[T1]{fontenc}
\usepackage{amsfonts}
\usepackage{algorithm}
\usepackage{algpseudocode}
\usepackage{xcolor}
\usepackage{enumitem}
\usepackage{cleveref}
\usepackage{graphicx}
\usepackage[export]{adjustbox}
\usepackage{subcaption}
\usepackage{wrapfig}
\usepackage{caption}
\usepackage{multirow}
\usepackage{makecell}
\usepackage{colortbl}
\usepackage[skins]{tcolorbox}
\usepackage{pifont}
\usepackage{titletoc}
\newcommand{\cmark}{\ding{51}}
\newcommand{\xmark}{\ding{55}}

\definecolor{darkblue}{rgb}{0,0,0.5}
\definecolor{discretecolor}{RGB}{11,83,150}
\definecolor{gaussiancolor}{RGB}{230,145,56}
\definecolor{argmaxcolor}{RGB}{154,0,0}
\definecolor{grayseq}{RGB}{120,120,120}
\definecolor{ourscolor}{RGB}{218,232,252}
\definecolor{lowuq}{RGB}{45,84,140}     
\definecolor{otherdec}{RGB}{145,92,44}  
\hypersetup{colorlinks=true, citecolor=darkblue, linkcolor=darkblue, urlcolor=darkblue}

\title{Locally Confident, Globally Stuck: The Quality-Exploration Dilemma in Diffusion Language Models}

\author{
\parbox{\textwidth}{
\textbf{Liancheng Fang}$^{1}$,
\textbf{Aiwei Liu}$^{2}$,
\textbf{Henry Peng Zou}$^{1}$,
\textbf{Yankai Chen}$^{3,4}$,
\textbf{Enze Ma}$^{1}$,
\textbf{Leyi Pan}$^{2}$, \\
\textbf{Chunyu Miao}$^{1}$,
\textbf{Wei-Chieh Huang}$^{1}$,
\textbf{Xue Liu}$^{3,4}$,
\textbf{Philip S. Yu}$^{1}$ \\[0.5em]
\textnormal{$^{1}$University of Illinois Chicago, $^{2}$Tsinghua University, $^{3}$MBZUAI, $^{4}$McGill University} \\
{\normalfont\hspace*{0.02em}\texttt{\{lfang87, psyu\}@uic.edu}}
}
}

\newtheorem{definition}{Definition}

\renewcommand{\KL}{\mathrm{KL}}

\renewcommand{\E}{\mathbb{E}}
\renewcommand{\1}{\mathbf{1}}

\newcommand{\bx}{\boldsymbol{x}}

\definecolor{brickred}{rgb}{0.8, 0.25, 0.33}
\definecolor{midnightblue}{rgb}{0.1, 0.1, 0.44}
\definecolor{oceanboatblue}{rgb}{0.0, 0.47, 0.75}
\definecolor{darkgreen}{RGB}{0,128,0}
\definecolor{lightlightgray}{RGB}{230, 230, 230}

\begin{document}

\ifcolmsubmission
\linenumbers
\fi

\maketitle

\begin{abstract}
Diffusion large language models (dLLMs) theoretically permit token decoding in arbitrary order, a flexibility that could enable richer exploration of reasoning paths than autoregressive (AR) LLMs. In practice, however, random-order decoding often hurts generation quality. To mitigate this, low-confidence remasking improves single-sample quality (e.g., Pass@$1$) by prioritizing confident tokens, but it also suppresses exploration and limits multi-sample gains (e.g., Pass@$k$), creating a fundamental \emph{quality--exploration dilemma}.
In this paper, we provide a unified explanation of this dilemma. We show that low-confidence remasking improves a myopic proxy for quality while provably constraining the entropy of the induced sequence distribution. To overcome this limitation, we characterize the optimal distribution that explicitly balances quality and exploration, and develop a simple \emph{Independent Metropolis--Hastings} sampler that approximately targets this distribution during decoding.
Experiments across a range of reasoning benchmarks including MATH500, AIME24/25, HumanEval, and MBPP show that our approach yields better exploration-quality tradeoff than both random and low-confidence remasking.
\end{abstract}

\section{Introduction}\label{sec:intro}
\begin{figure}[h]
  \centering
  \includegraphics[width=\textwidth]{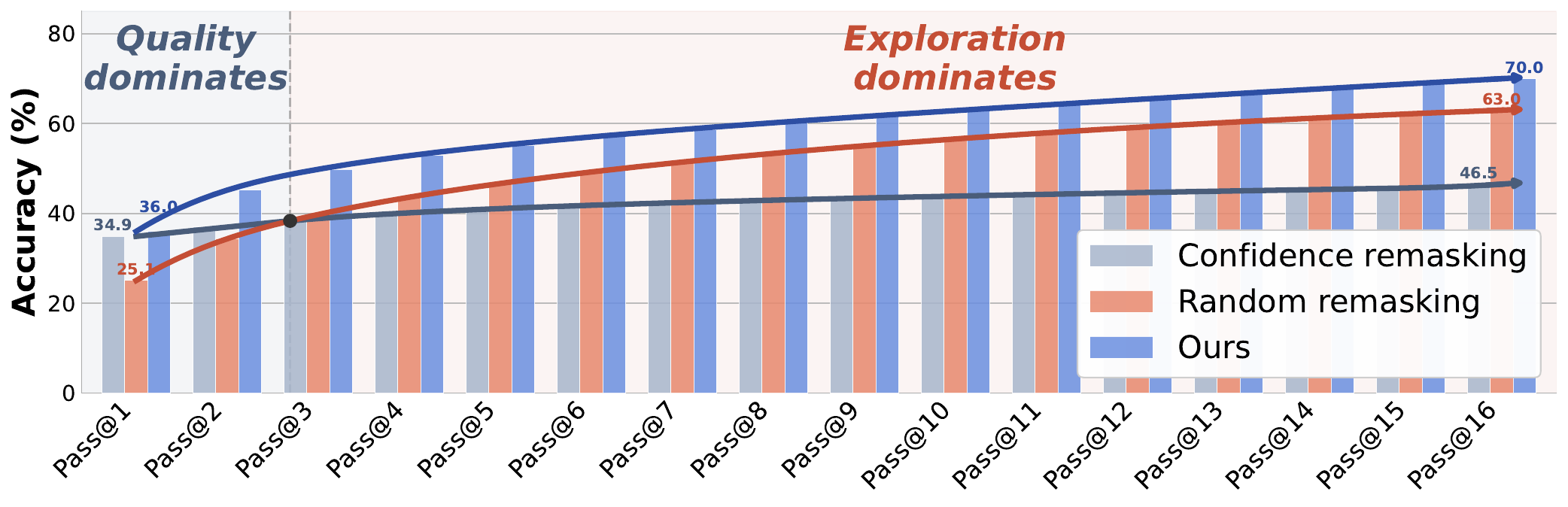}
  \caption{\textbf{The quality--exploration dilemma in dLLM decoding.}
  Confidence remasking achieves high sample quality (Pass@$1$) but plateaus Pass@$k$ quickly due to limited exploration. Conversely, random remasking promotes exploration but degrades individual sample quality. Our \emph{global tempering} reconciles this trade-off, establishing a new Pareto frontier with superior Pass@$1$ and Pass@$16$ performance.}
  \label{fig:teaser}
\end{figure}

Diffusion large language models (dLLMs) generate text by iteratively denoising an entire token sequence in parallel~\citep{Lou2023DiscreteDM,shi2024simplified,ou2024your,sahoo2024simple,nie2025large,ye2025dream,song2025seed}. Unlike autoregressive (AR) LLMs~\citep{radford2018improving,raffel2020exploring,tom2020gpt}, which commit tokens irrevocably in a fixed left-to-right order, dLLMs may finalize tokens at arbitrary positions. This flexibility is intrinsically aligned with the non-linear nature of complex reasoning, where pivotal decisions need not adhere to a strict left-to-right progression~\citep{bachmann2024pitfalls,nagarajan2025roll,kim2025train,yang2025powerful,trainin2026discrete}.

In practice, however, realizing this advantage has proven difficult. Recent work~\citep{ni2026flexibility,shen2026improving,chen2025beyond,fu2025bits,lee2025lookahead} points to a fundamental \emph{quality--exploration dilemma} (\Cref{fig:teaser}): low-confidence remasking strategies, which commit the model's most certain predictions first~\citep{wu2025fast,kim2025train,ben2025accelerated}, improve single-sample quality (Pass@$1$) but saturate quickly under repeated sampling (Pass@$k$)~\citep{ni2026flexibility,shen2026improving}, revealing a severe exploration bottleneck. Random remasking exhibits the opposite behavior: it explores more broadly, yet produces weaker individual samples. This tension motivates our central question:
\begin{center}
\emph{Can a dLLM decoding strategy achieve high per-sample quality without collapsing exploration?}
\end{center}

In this paper, we do so by deriving the target distribution that optimally balances quality and exploration, then designing a new decoding strategy that approximately samples from the optimal distribution. 
The strategy favors globally promising sequences through a lookahead correction that adjusts each local token choice according to \emph{how promising the resulting space of completions is}.
We show that this strategy yields a better exploration-quality tradeoff than both uncertainty-based remasking and random remasking. Our contributions are as follows:
\begin{enumerate}[leftmargin=*]
    \item We provide a unified explanation of the exploration bottleneck in uncertainty-based decoding, where we show that the shared mode-seeking behavior of uncertainty-based decoding improves a myopic quality proxy while imposing a formal entropy cap on the induced sequence distribution.
    \item We formalize the quality--exploration trade-off as an entropy-regularized optimization problem over the joint distribution and characterize the unique optimal solution. 
    \item We design a practical Markov chain Monte Carlo (MCMC) algorithm based on \emph{Independent Metropolis--Hastings} that efficiently approximates this target distribution during decoding using a tractable lookahead correction.
    \item Extensive experiments with LLaDA~\citep{nie2025large} and WeDLM~\citep{liu2025wedlm} on reasoning benchmarks (MATH500, AIME, HumanEval, MBPP) demonstrate that our approach consistently achieves superior exploration--quality trade-offs.
\end{enumerate}

\section{Preliminaries}
\label{sec:prelim}

\textbf{Diffusion Language Models.}
Let $\mathcal{V}$ be a finite vocabulary and $\mathtt{[M]}$ a special mask token. We consider
sequences of length $L$ with index set $[L] \coloneqq \{1, \ldots, L\}$. A diffusion
language model generates samples by learning to reverse a \textit{forward process} that
corrupts a clean sequence $\bx_0 \in \mathcal{V}^L$ into a partially masked
sequence $\bx_t \in (\mathcal{V} \cup \{\mathtt{[M]}\})^L$ by independently
masking each token with probability $t \in [0, 1]$ (linear masking schedule).
The \emph{reverse process} is parameterized by a mask predictor $p_\theta(\cdot \mid
\bx_t)$ that independently predicts all masked tokens from the corrupted input.
It is trained to maximize an evidence lower bound (ELBO) on the data log-likelihood, which
reduces to a weighted denoising cross-entropy on masked positions~\citep{shi2024simplified,
ou2024your,zheng2024masked}:
\begin{equation} 
  \mathcal{L}(\theta)
  \;\coloneqq\;
  -\mathbb{E}_{\bx_0,\, t,\, \bx_t}
  \left[
    \frac{1}{t}
    \sum_{i \in [L]}
    \mathbf{1}[x_t^i = \mathtt{[M]}]
    \;\log p_\theta(\bx_0^i \mid \bx_t)
  \right].
  \label{eq:dllm_weighted_ce}
\end{equation}

The sampling process iteratively refines a fully masked sequence. At each step, the model predicts marginal distributions for all masked positions simultaneously, conditioned on the partially revealed sequence. It then unmasks a subset of these positions by sampling from the marginals independently. Repeating this process gradually yields a complete sequence. 
Two main strategies determine this subset: 
\begin{enumerate}[leftmargin=*]
\item \textit{Random remasking}: selects positions to unmask uniformly at random and remask the remaining masked positions, which is theoretically grounded in $\tau$-leaping~\citep{campbell2022continuous,Lou2023DiscreteDM,ou2024your}.
\item \textit{Uncertainty-based remasking}: selects unmask positions with high token or distribution uncertainty scores. See \Cref{tab:uncertainty-heuristics} for a taxonomy of common uncertainty heuristics.
\end{enumerate}

\textbf{Pass@$k$ as Exploration Metric.}
Pass@$k$ measures the probability of sampling at least one correct solution among $k$ independent generations, serving as a standard proxy for a model's exploration capability~\citep{yue2025does}. Given $c$ correct solutions out of $n$ total samples, its unbiased estimator is given by~\citep{chen2021evaluating}:
\begin{equation} \label{eq:passk}
  \text{Pass@}k = \mathbb{E}\!\left[1 - \frac{\binom{n-c}{k}}{\binom{n}{k}}\right].
\end{equation}
A low Pass@$k$ indicates a fundamental barrier in exploring valid reasoning trajectories.

\begin{table}[t]
\centering
\caption{\textbf{Taxonomy of uncertainty heuristics for uncertainty-based remasking.} We categorize heuristics by their distributional scoring functions ($p_i(\cdot) \coloneqq p(\cdot \mid s_t, i)$) and operational paradigms (\emph{Sample-then-Filter} vs.\ \emph{Rank-then-Sample}). The \emph{$\delta$-gating} column establishes the equivalent confidence lower bound $\max_v p_i(v) \ge 1-\delta$ for each criterion. See \Cref{app:delta-gating} for detailed definitions and formal derivations of these bounds.}
\label{tab:uncertainty-heuristics}
\vspace{0.5em}
\resizebox{\textwidth}{!}{%
\begin{tabular}{@{}llll@{}}
\toprule
\textbf{Heuristic} & \textbf{Commit criterion} & \textbf{Mode} & \textbf{$\delta$-gating} \\
\midrule
Confidence~\citep{nie2025large,ye2025dream}
  & $\max_{v}\, p_i(v) \ge 1-\delta$ & Sample-then-Filter 
  & $\delta$ (directly) \\
\midrule
Entropy~\citep{ben2025accelerated,fu2025bits}
  & $H(p_i) \le \varepsilon$ & Rank-then-Sample
  & $\delta = \delta_V(\varepsilon)$ \\
\midrule
Margin~\citep{kim2025train,hong2025improving}
  & $p_i(v_1) - p_i(v_2) \ge \gamma$ & Rank-then-Sample
  & $\delta = \frac{(|V|-1)(1-\gamma)}{|V|}$ \\
\bottomrule
\end{tabular}%
}
\end{table}

\section{Formalizing the Quality--Exploration Dilemma}
\label{sec:confidence-gating}

To formalize the quality--exploration dilemma, we begin by unifying common uncertainty heuristics under the notion of \emph{confidence gating}.
\begin{definition}[Confidence gating]
\label{def:confidence-gating}
Let \(V\) denote the token vocabulary and \(s_t\) the decoder state at step \(t\). For a threshold parameter \(\delta \in (0,1)\), a decoder is \((1-\delta)\)-gated if, whenever it commits a token, the commit distribution satisfies
\begin{equation}
\max_{v \in V} p(v \mid s_t) \ge 1 - \delta.
\end{equation}
\end{definition}
Intuitively, confidence gating permits a token to be committed only when its marginal distribution is sufficiently peaked. Although stated in terms of confidence, this definition naturally subsumes other common heuristics: entropy and margin-based criteria all reduce to this form under appropriate choices of \(\delta\) (see \Cref{tab:uncertainty-heuristics} for details).

\subsection{The Local Benefit: Improving Quality via Myopic Optimization}
To elucidate why uncertainty-based decoding enhances single-sample performance (e.g., Pass@1), we analyze its implicit objective: the expected log loss under a reference model. Let $q_{\mathrm{gen}}$ denote the joint distribution over generated sequences, $\sigma \in S_L$ the decoding trajectory (where $S_L$ is the permutation group over $1,\dots,L$), and $p_{\mathrm{ref}}$ a scoring model. We define the trajectory-dependent generation loss as:
\begin{equation}
\mathcal{L}_{\mathrm{gen}}(q_{\mathrm{gen}}; p_{\mathrm{ref}}; \sigma)
:=
\mathbb{E}_{\vx\sim q_{\mathrm{gen}}}
\left[
\frac{1}{L}\sum_{t=1}^L
-\log p_{\mathrm{ref}}(\vx_{\sigma_t}\mid \vx_{\sigma_{<t}})
\right].
\end{equation}
Exponentiating this expected loss yields \emph{generative perplexity} (GenPPL), a standard metric for evaluating the generative quality of diffusion models~\citep{Lou2023DiscreteDM,arriola2025block,sahoo2025diffusion}:
\begin{equation}
\label{eq:genppl}
\operatorname{GenPPL}(q_{\mathrm{gen}}; p_{\mathrm{ref}}; \sigma)
:=
\exp\!\bigl(\mathcal{L}_{\mathrm{gen}}(q_{\mathrm{gen}}; p_{\mathrm{ref}}; \sigma)\bigr).
\end{equation}
Intuitively, GenPPL measures the sequence-level surprisal under the reference model. By setting $p_{\mathrm{ref}} = p_\theta$, we evaluate the model's self-consistency along the generation trajectory. Under this self-scoring regime, the expected generation loss decomposes into a sum of step-wise conditional entropies. Specifically, when the decoder commits a token by sampling from a chosen position's marginal distribution $p_\theta(\cdot \mid s_t)$, the expected one-step loss incurred is exactly the entropy of that distribution: $\mathbb{E}_{x}[-\log p_\theta(x \mid s_t)] = H(p_\theta(\cdot \mid s_t))$. Consequently, confidence gating---which prioritizes highly peaked, low-entropy distributions---acts as a myopic, greedy heuristic to minimize $\mathcal{L}_{\mathrm{gen}}$. \Cref{prop:genppl-upper} formalizes this mechanism, proving that bounding the step-wise entropy strictly bounds the global generative perplexity.

\begin{restatable}[Generative perplexity upper bound under confidence gating]{proposition}{genpplupperproposition}
\label{prop:genppl-upper}
Assume the decoder is \((1-\delta)\)-gated, the committed token at each step is sampled from the decoder's commit distribution, and the scoring model is the decoder itself, i.e., \(p_{\mathrm{ref}} = p_\theta\). Then for any decoding trajectory \(\sigma\),
\[
\operatorname{GenPPL}(q_{\mathrm{gen}};\, p_\theta;\, \sigma) \le \exp\!\bigl(h_V(\delta)\bigr),
\]
where \(h_V(\delta) = h_b(\delta) + \delta \log(|V|-1)\) and \(h_b(\delta) = -\delta \log \delta - (1-\delta)\log(1-\delta)\).
\end{restatable}

The bound $\exp(h_V(\delta))$ corresponds precisely to the maximum entropy of any vocabulary distribution whose largest probability mass is at least $1-\delta$. By controlling the worst-case one-step uncertainty, confidence gating effectively bounds the accumulated self-cross-entropy along the generation trajectory. This provides a principled explanation for why uncertainty-based heuristics improve single-sample quality. Nevertheless, because each commit decision fundamentally alters the context for future steps, this approach remains a \emph{myopic} surrogate rather than a direct optimization of the global sequence-level objective.

\subsection{The Global Cost: Suppressing Exploration via Entropy Collapse}
\label{sec:entropy-cap}
While the mode-seeking behavior of confidence gating improves single-sample quality, it inherently restricts sequence-level diversity, which ultimately hurts multi-sample performance (e.g., Pass@$k$). \Cref{prop:entropy-cap} formalizes this limitation by bounding the sequence entropy $H(X)$ and the effective branching factor~\citep{jm3}.

\begin{restatable}[Entropy cap under confidence gating]{proposition}{entropycapproposition}
\label{prop:entropy-cap}
Assume the dLLM decoder is \((1-\delta)\)-gated. Then the sequence entropy satisfies
\[
H(X) \le L \cdot h_V(\delta),
\]
where \(h_V(\delta) = h_b(\delta) + \delta \log(|V|-1)\) and \(h_b(\delta) = -\delta \log \delta - (1-\delta)\log(1-\delta)\) is the binary entropy function. Equivalently, the effective branching factor, defined as the per-token perplexity of the induced sequence distribution,
\[
B_{\mathrm{eff}} := \exp(H(X)/L),
\]
is upper bounded as
\[
B_{\mathrm{eff}} \le \exp\!\bigl(h_V(\delta)\bigr).
\]
\end{restatable}

\Cref{prop:entropy-cap} establishes a hard \emph{entropy budget} that depends only on
the gating threshold~$\delta$ and the vocabulary size~$|V|$. The induced
quantity~$B_{\mathrm{eff}}$, which measures the geometric-mean number of
effective tokens available at each decoding step, is sharply constrained under
strong gating. As a concrete instantiation, setting $|V|=5\times10^{4}$ and
$\delta=0.05$ yields $B_{\mathrm{eff}}< 2.1$, indicating that the
decoder concentrates almost all probability mass on fewer than three candidates
per step. This tight ceiling explains the poor Pass@$k$ scaling: even if the model internally supports multiple correct reasoning paths, a strongly gated decoder cannot simultaneously explore these valid modes.


\section{Principled Exploration via Global Tempering}
\label{sec:method}

\subsection{The Optimal Target Distribution}
\label{sec:global_tempering}
Motivated by the limitations of local heuristics discussed in \Cref{sec:confidence-gating}, we now formalize the quality--exploration trade-off at the level of complete-sequence distributions. We define generation quality as the expected log-likelihood under the base model, $\mathbb{E}_{x \sim p}[\log q(x)]$, and exploration as the Shannon entropy $H(p)$ of the sequence distribution. Rather than relying on step-wise proxies, we optimize directly over the joint distribution $p$, yielding the entropy-regularized objective:
\begin{equation}
\label{eq:obj_maxent}
\max_{p \in \Delta(\mathcal{X})} \underbrace{\colorbox{cyan!10}{$\alpha \, \mathbb{E}_{\vx \sim p}[\log q(\vx)]$}}_{\text{quality}} + \underbrace{\colorbox{red!10}{$H(p)$}}_{\text{exploration}},
\end{equation}
where $\alpha \ge 0$ controls the trade-off: larger $\alpha$ places more mass on high-likelihood sequences, while the entropy term discourages mode collapse.

\begin{restatable}[Optimality of the power distribution]{proposition}{poweroptproposition}
\label{prop:power-opt}
For any $\alpha \geq 0$, the unique optimizer of~\eqref{eq:obj_maxent} is the power distribution
\begin{equation*}
  p_\alpha^\star(\bx)=\frac{q(\bx)^\alpha}{Z_\alpha},
  \qquad
  Z_\alpha=\sum_{\bx\in\mathcal{X}} q(\bx)^\alpha.
\end{equation*}
\end{restatable}
Formally, $p_\alpha^\star$ applies a \emph{global} temperature scaling to the joint sequence distribution $q(\bx)$, flattening ($\alpha < 1$) or sharpening ($\alpha > 1$) the sequence-level distribution. Unlike conventional local logit tempering, $p_\alpha^\star$ operates over the entire sequence. We now address the resulting algorithmic challenge: efficiently sampling from this globally tempered distribution.

\subsection{Corrected Conditionals: Exact Form and Tractable Approximation}
\label{sec:corrected_conditionals}

Sampling from the globally tempered distribution $p^\star_\alpha(\bx) \propto q(\bx)^\alpha$ presents two major obstacles in the dLLM setting. First, the normalizing constant $Z_\alpha$ is intractable to compute over the exponentially large sequence space $\mathcal{X} = \mathcal{V}^L$. While autoregressive LLMs can circumvent this issue using Markov chain Monte Carlo (MCMC) samplers that only require evaluating unnormalized sequence likelihoods~\citep{karan2025reasoning}, dLLMs face a second, more severe hurdle: their exact sequence likelihood $q(\bx)$ is itself intractable~\citep{shi2024simplified,pan2025d}. Consequently, standard sequence-level MCMC techniques cannot be directly applied.

To bypass the need for evaluating intractable sequence-level likelihoods, we shift our perspective from the global joint distribution to the local one-step generation. Fundamentally, the generative process of a dLLM samples this joint distribution by iteratively sampling from a sequence of marginal conditionals. A natural question thus arises: for a global tempered joint distribution, what is the exact form of the resulting marginal conditionals at each decoding step? We can analytically derive this marginal conditional below:

\begin{restatable}[Corrected conditional for global tempering]{proposition}{exactcorrectedlemma}
\label{lem:exact-corrected-conditional}
Let \(q\) be a distribution over \(\mathcal{X} = \mathcal{V}^L\) with full support, and fix \(\alpha > 0\). Let \(s = (A, x_A)\) denote a partially decoded state with committed positions \(A \subseteq [L]\) and remaining positions \(R(s) := [L] \setminus A\). For any uncommitted position \(i \in R(s)\) and token \(v \in \mathcal{V}\), let \(\boldsymbol{\ell}_i(s)\) be the logit vector such that \(q(\vx_i = v \mid s) = \operatorname{softmax}(\boldsymbol{\ell}_i(s))_v\). Letting \(s' = s \oplus (i, v)\) denote the updated state after committing token \(v\) at position \(i\), define
\begin{equation*}
{\color{argmaxcolor}\Delta_{i,v}(s)}
:=
\log
\sum_{\vx_{R(s')} \in V^{|R(s')|}}
q(\vx_{R(s')} \mid s')^\alpha.
\end{equation*}
Then the conditional of \(p_\alpha^\star(x) \propto q(x)^\alpha\) at position \(i\) is
\begin{equation}
p_\alpha^\star(\vx_i = v \mid s)
=
\frac{
\exp\!\bigl({\color{discretecolor}\alpha \boldsymbol{\ell}_{i,v}(s)} + {\color{argmaxcolor}\Delta_{i,v}(s)}\bigr)
}{
\sum_{u \in V}
\exp\!\bigl({\color{discretecolor}\alpha \boldsymbol{\ell}_{i,u}(s)} + {\color{argmaxcolor}\Delta_{i,u}(s)}\bigr)
}.
\labelonce{eq:exact_corrected_conditional}
\end{equation}
\end{restatable}

\Cref{eq:exact_corrected_conditional} features two interpretable terms: a \textbf{local tempering} term {\color{discretecolor}\(\alpha\boldsymbol{\ell}_{i,v}(s)\)} and a \textbf{suffix lookahead correction} {\color{argmaxcolor}\(\Delta_{i,v}(s)\)}, which measures the total tempered probability mass of the completions reachable after committing \(v\). 
Crucially, {\color{argmaxcolor}\(\Delta_{i,v}(s)\)} does not simply favor the most locally confident token. Instead, it favors tokens that preserve more \emph{globally promising continuations}. This is why it can mitigate the entropy cap of confidence-based decoding: confidence gating suppresses exploration by forcing low entropy at the current step, whereas the lookahead correction evaluates each token by the value of the continuation space it leaves available, avoiding premature collapse onto a single reasoning path.

\noindent\textbf{Tractable approximation via mean-field factorization.}
Evaluating the exact lookahead correction {\color{argmaxcolor}\(\Delta_{i,v}(s)\)} is computationally prohibitive due to the exponential cardinality of the suffix space. However, the mask-prediction objective of dLLMs inherently models the sequence via conditionally independent marginals given the current state (\Cref{eq:dllm_weighted_ce}). This structural property motivates a mean-field approximation~\citep{zhao2025d1}, allowing us to factorize the intractable suffix joint as $q(\vx_{R(s')} \mid s') \approx \prod_{j \in R(s')} q(\vx_j \mid s')$. This factorization decouples the $\alpha$-power summation along the sequence length:
\begin{equation*}
{\color{argmaxcolor}\widehat{\Delta}_{i,v}(s)}
:=
\sum_{j \in R(s')}
\log
\Biggl(
\sum_{u \in V} q(\vx_j = u \mid s')^\alpha
\Biggr),
\qquad s' = s \oplus (i, v).
\end{equation*}
Consequently, the surrogate correction $\widehat{\Delta}_{i,v}(s)$ becomes readily computable. It requires only a \emph{single} network evaluation at the proposed state $s'$ to obtain the requisite marginals. Plugging this approximation back into~\eqref{eq:exact_corrected_conditional} yields a one-step sampling target:
\begin{equation}
\pi_i(v \mid s)
\propto
\exp\!\bigl({\color{discretecolor}\alpha \boldsymbol{\ell}_{i,v}(s)} + {\color{argmaxcolor}\widehat{\Delta}_{i,v}(s)}\bigr),
\qquad v \in V.
\label{eq:approx_position_target}
\end{equation}

\subsection{Independent Metropolis--Hastings with Batched Lookahead}
\label{sec:imh}

\begin{wrapfigure}{R}{0.50\textwidth}
\vspace{-2.2em}
\begin{minipage}{0.50\textwidth}
\begin{algorithm}[H]
\caption{Batched IMH for one-token corrected sampling}
\label{alg:imh-token}
\begin{algorithmic}[1]
\Require State \(s\); position \(i \in R(s)\); number of proposals \(T\); exponent \(\alpha\)
\Ensure A token approximately distributed as \(\pi_i(\cdot \mid s)\) in \eqref{eq:approx_position_target}
\State Sample \(x \sim r_i(\cdot \mid s)\) using \eqref{eq:proposal}
\State Draw i.i.d. proposals \(y_1, \ldots, y_T \sim r_i(\cdot \mid s)\)
\State Compute \({\color{argmaxcolor}\widehat{\Delta}_{i,x}(s)}\) and \({\color{argmaxcolor}\widehat{\Delta}_{i,y_t}(s)}\) for \(t \in \{1, \dots, T\}\) in one batch
\For{\(t = 1, \ldots, T\)}
    \State \(a \leftarrow \min\{1, \exp\bigl({\color{argmaxcolor}\widehat{\Delta}_{i,y_t}(s)} - {\color{argmaxcolor}\widehat{\Delta}_{i,x}(s)}\bigr)\}\)
    \State Sample \(u \sim \mathrm{Uniform}(0,1)\)
    \If{\(u < a\)}
        \State \(x \leftarrow y_t\)
        \State \({\color{argmaxcolor}\widehat{\Delta}_{i,x}(s)} \leftarrow {\color{argmaxcolor}\widehat{\Delta}_{i,y_t}(s)}\)
    \EndIf
\EndFor
\State \Return \(x\)
\end{algorithmic}
\end{algorithm}
\end{minipage}
\vspace{-1em}
\end{wrapfigure}

Given the pointwise evaluable one-step target $\pi_i(\cdot \mid s)$, we employ Markov Chain Monte Carlo (MCMC) to sample from this unnormalized distribution. Specifically, we propose an independent Metropolis--Hastings (IMH) sampler~\citep{hastings1970monte}. For a fixed decoding state $s$ and position $i$, we define the independence proposal as the locally tempered logits without correction:
\begin{equation}\label{eq:proposal}
  r_i(v \mid s) \;:=\; \operatorname{softmax}\!\bigl({\color{discretecolor}\alpha\,\boldsymbol{\ell}_i(s)}\bigr)_v\,.
\end{equation}

This proposal distribution is chosen for two primary reasons. First, we empirically find that it captures the dominant probability mass of the target distribution (see \Cref{tab:imh-convergence} for validation). Second, because the target factorizes as $\pi_i(v \mid s) \propto r_i(v \mid s)\exp({\color{argmaxcolor}\widehat{\Delta}_{i,v}(s)})$, the local ${\color{discretecolor}\alpha\boldsymbol{\ell}}$ terms exactly cancel within the Metropolis--Hastings ratio (see \Cref{app:imh-derivation} for the detailed derivation), yielding a highly streamlined acceptance probability:
\begin{equation}\label{eq:acceptance}
  A(x \to y)
  \;=\;
  \min\!\Bigl\{1,\;\exp\!\bigl({\color{argmaxcolor}\hat\Delta_{i,y}(s)} - {\color{argmaxcolor}\hat\Delta_{i,x}(s)}\bigr)\Bigr\}\,.
\end{equation}
Consequently, the IMH sampler generates candidates from $r_i$ and accepts them via~\Cref{eq:acceptance}, relying solely on the difference in suffix corrections without requiring vocabulary-wide normalization. The full algorithm is presented in \Cref{alg:imh-token}.

\textbf{Batched evaluation.}
A primary computational advantage of our IMH sampler is that the proposal distribution $r_i$ depends strictly on the outer context $(s, i)$, entirely decoupled from the Markov chain's current state. This permits all $T$ proposals and their associated lookahead corrections to be evaluated concurrently in a \emph{single} batched forward pass. Consequently, IMH introduces zero sequential overhead compared to original dLLM decoding.

\textbf{MCMC mixing.} Traditional MCMC algorithms frequently suffer from slow mixing when deployed in high-dimensional discrete spaces. Our IMH sampler circumvents this bottleneck through two structural advantages. First, by formulating the Markov chain to target a 1D categorical distribution over the vocabulary $\mathcal{V}$, we completely bypass the curse of dimensionality. Second, since our proposal distribution $r_i$ tightly approximates the target by incorporating dominant local logits, standard MCMC theory guarantees rapid geometric convergence~\citep{mengersen1996rates}. We empirically validate both points in \Cref{sec:experiments}.

\section{Experiments}
\label{sec:experiments}

\begin{table}[t]
\centering
\caption{\textbf{Pass@$\bm{1}$ and Pass@$\bm{k}$ across benchmarks and models.} \textbf{Bold} indicates the best result for a given metric. LLaDA-8B results on AIME24/25 are omitted as they are close to 0.}
\label{tab:main}
\setlength{\tabcolsep}{5pt}
\renewcommand{\arraystretch}{1.05}
\small
\resizebox{\textwidth}{!}{%
\begin{tabular}{@{}ll *{3}{c} *{4}{c}@{}}
\toprule
& & \multicolumn{3}{c}{\textbf{Math reasoning}}
& \multicolumn{4}{c}{\textbf{Code generation}} \\
\cmidrule(lr){3-5} \cmidrule(l){6-9}
\textbf{Method} & \textbf{Metric}
& MATH500 & AIME'24 & AIME'25 & HumanEval & HumanEval+ & MBPP & MBPP+ \\
\midrule
\rowcolor{blue!8}\multicolumn{9}{c}{\textbf{WeDLM-8B}} \\
\midrule
Entropy$^\dagger$
  & Pass@1 & 0.528 & 0.093 & 0.071 & 0.732 & 0.684 & \textbf{0.784} & \textbf{0.677} \\
  & Pass@k & \textbf{0.875} & 0.310 & 0.291 & 0.950 & 0.911 & 0.916 & 0.739 \\
\cmidrule{1-9}
Confidence
  & Pass@1 & 0.528 & 0.079 & \textbf{0.074} & 0.725 & 0.678 & 0.707 & 0.598 \\
  & Pass@k & 0.851 & 0.323 & 0.292 & 0.905 & 0.857 & 0.834 & 0.717 \\
\cmidrule{1-9}
Margin
  & Pass@1 & 0.393 & 0.058 & 0.031 & 0.584 & 0.537 & 0.589 & 0.490 \\
  & Pass@k & 0.838 & 0.296 & 0.269 & 0.903 & 0.843 & 0.765 & 0.643 \\

\midrule
IMH (ours)
  & Pass@1 & \textbf{0.540} & \textbf{0.095} & \textbf{0.074} & \textbf{0.745} & \textbf{0.696} & 0.776 & 0.664 \\
  & Pass@k & \textbf{0.875} & \textbf{0.344} & \textbf{0.344} & \textbf{0.962} & \textbf{0.933} & \textbf{0.938} & \textbf{0.836} \\
\midrule
\rowcolor{blue!8}\multicolumn{9}{c}{\textbf{LLaDA-8B}} \\
\midrule
Confidence$^\dagger$
  & Pass@1 & 0.349 & - & - & \textbf{0.410} & \textbf{0.372} & \textbf{0.484} & \textbf{0.404} \\
  & Pass@k & 0.465 & - & - & 0.573 & 0.536 & 0.645 & 0.521 \\
\cmidrule{1-9}
Entropy
  & Pass@1 & 0.294 & - & - & 0.387 & 0.344 & 0.444 & 0.387 \\
  & Pass@k & 0.606 & - & - & 0.670 & 0.603 & 0.751 & 0.656 \\
\cmidrule{1-9}
Margin
  & Pass@1 & 0.286 & - & - & 0.350 & 0.315 & 0.451 & 0.388 \\
  & Pass@k & 0.594 & - & - & 0.610 & 0.573 & 0.716 & 0.589 \\
\cmidrule{1-9}
Random
  & Pass@1 & 0.251 & - & - & 0.196 & 0.176 & 0.279 & 0.251 \\
  & Pass@k & 0.630 & - & - & 0.549 & 0.524 & 0.685 & 0.605 \\
\cmidrule{1-9}
IMH (ours)
  & Pass@1 & \textbf{0.360} & - & - & 0.385 & 0.345 & 0.469 & 0.399 \\
  & Pass@k & \textbf{0.700} & - & - & \textbf{0.695} & \textbf{0.640} & \textbf{0.778} & \textbf{0.672} \\
\bottomrule
\addlinespace[2pt]
\multicolumn{9}{@{}l}{\footnotesize $^\dagger$\,Default decoding method of the respective model.} 
\end{tabular}%
}
\end{table}

\textbf{Experimental Setup.} We evaluate our method on reasoning-intensive tasks across two domains: mathematical problem-solving (MATH500~\citep{hendrycks2020measuring}, AIME 2024~\citep{aime24}, AIME 2025~\citep{aime25}) and code generation (HumanEval~\citep{chen2021evaluating}, MBPP~\citep{austin2021program}). Our evaluation framework incorporates two architecturally distinct dLLMs: the fully bidirectional LLaDA-8B-Instruct~\citep{nie2025large} and the block-diffusion WeDLM-8B~\citep{liu2025wedlm}. We benchmark against random remasking and standard confidence-based heuristics (confidence, entropy, and margin; see \Cref{tab:uncertainty-heuristics} for definitions). For WeDLM, random remasking is excluded due to its inherent left-to-right generation bias~\citep{liu2025wedlm}, and we restrict its suffix lookahead correction to a 16-token window to align with its block-diffusion mechanics. We adopt Pass@$k$ (\Cref{eq:passk}) as our primary metric. To guarantee a statistically robust, low-variance estimation, we systematically oversample: generating $n=32$ samples for $k \le 16$ with LLaDA, and $n=128$ for $k \le 32$ with WeDLM (see \Cref{app:pass_at_k_variance} for variance analysis). Finally, to ensure a strictly hyperparameter-free comparison, all baselines decode exactly one token per step. More experimental details are provided in \Cref{app:experimental-details}.

\begin{figure}[t!]
  \centering
  \includegraphics[width=\textwidth]{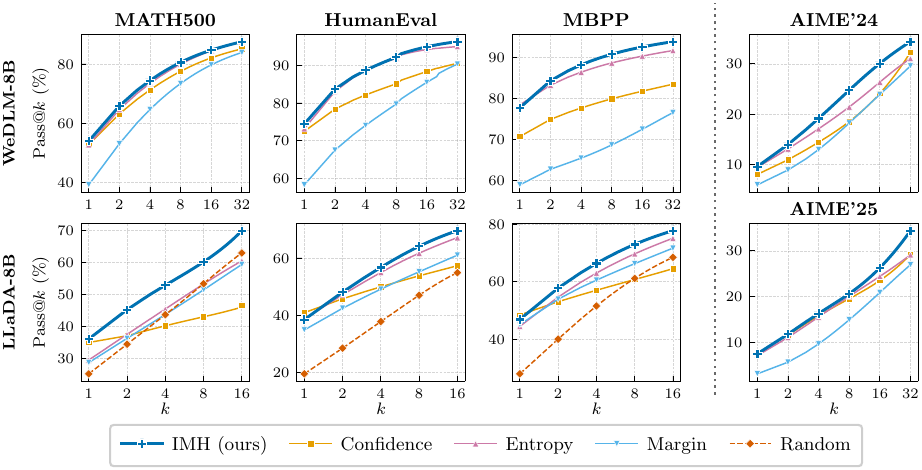}
  \caption{\textbf{Pass@$\bm{k}$ scaling curves.} \textbf{(Left)} Performance on MATH500, HumanEval, and MBPP for WeDLM-8B (top) and LLaDA-8B (bottom). \textbf{(Right)} WeDLM-8B on AIME'24/25.}
  \label{fig:passk-scaling}
\end{figure}

\textbf{Benchmarks.}
In \Cref{tab:main}, we report Pass@$1$ and Pass@$k$ across all benchmarks using a shared local temperature of 0.5~\citep{ye2024diffusion}. This reflects the practical setting where users select a specific 
temperature and allow for direct comparison under controlled conditions. 
First, IMH consistently delivers the strongest exploratory performance, attaining the best Pass@$k$ on nearly every task for both WeDLM and LLaDA. \Cref{fig:passk-scaling} reinforces this conclusion by showing the full Pass@$k$ scaling curves.
Notably, on the more challenging AIME24/25 benchmarks, IMH substantially outperforms WeDLM's default entropy heuristic in both Pass@$1$ and Pass@$k$.
Second, for LLaDA, IMH offers the most favorable overall trade-off between quality and exploration. On MATH500, it achieves the best results on both metrics. On HumanEval and MBPP, while low-confidence remasking attains slightly higher Pass@$1$, this marginal advantage comes at a substantial cost in Pass@$k$, making IMH clearly preferable when broader exploration is desired.

\textbf{The Quality-Diversity Pareto Frontier.}
We further validate the robustness of this exploration-exploitation trade-off across the various hyperparameters. \Cref{fig:combined-analysis} (left) plots Pass@$1$ versus Pass@$k$ on MATH500 by sweeping the temperature parameters ($\alpha$ for IMH, $\tau$ for baselines). These trajectories reveal that global tempering strictly dominates local strategies. While local methods struggle to balance exploiting high-likelihood paths with exploring diverse reasoning, IMH consistently pushes the Pareto frontier outward. This empirical superiority directly corroborates the theoretical optimality of the power distribution (\Cref{prop:power-opt}), establishing global tempering as a fundamentally more effective mechanism for navigating the exploration-exploitation trade-off.

\textbf{Expanding the Reasoning Boundary.}
While aggregate Pass@$k$ metrics demonstrate overall superiority, they cannot distinguish whether the improvements stem from merely increasing reliability on simpler tasks or from achieving genuine breakthroughs on previously intractable problems. To isolate these effects, we stratify AIME 2024 problems by difficulty following \citet{sun2025climbing}. As shown in \Cref{fig:combined-analysis} (middle), performance on the \emph{Easy} tier is largely saturated across all methods. However, as difficulty escalates, the performance gap gradually widens. Crucially, IMH yields the most pronounced gains on the \emph{Hard} tier, decisively outperforming the strongest baselines. This confirms that the trajectory-level diversity induced by the lookahead correction does not merely increase reliability on solvable problems, but actively expands the model's \emph{reasoning boundary}, enabling it to navigate the complex search spaces of the most challenging mathematical tasks. We provide a detailed case study in \Cref{app:case-study}, showcasing specific hard problems successfully solved by IMH that completely eluded all baseline methods.

\textbf{Trajectory Similarity Analysis.}
To understand \emph{why} IMH succeeds where local baselines fail, we analyze the semantic diversity of the generated reasoning paths. For each AIME 2024 problem, we sample one trajectory per method and employ an LLM judge to summarize the problem-solving strategy and score pairwise similarity (details in \Cref{app:trajectory-similarity}). The resulting similarity matrix (\Cref{fig:combined-analysis} (right)) exhibits a distinct block structure: the three local baselines (Conf, Margin, Entropy) are highly correlated with one another (scores 73--83). This indicates that confidence-based remasking heuristics tend to collapse into nearly identical reasoning paths, regardless of the specific uncertainty metric employed. In contrast, IMH trajectories exhibit significantly lower similarity to all baselines (scores 68--72). This demonstrates that IMH successfully induces qualitatively distinct reasoning strategies, rather than merely injecting local perturbations into the same dominant solution trajectory.

\begin{figure}[t!]
  \centering
  \includegraphics[width=0.34\textwidth,valign=t]{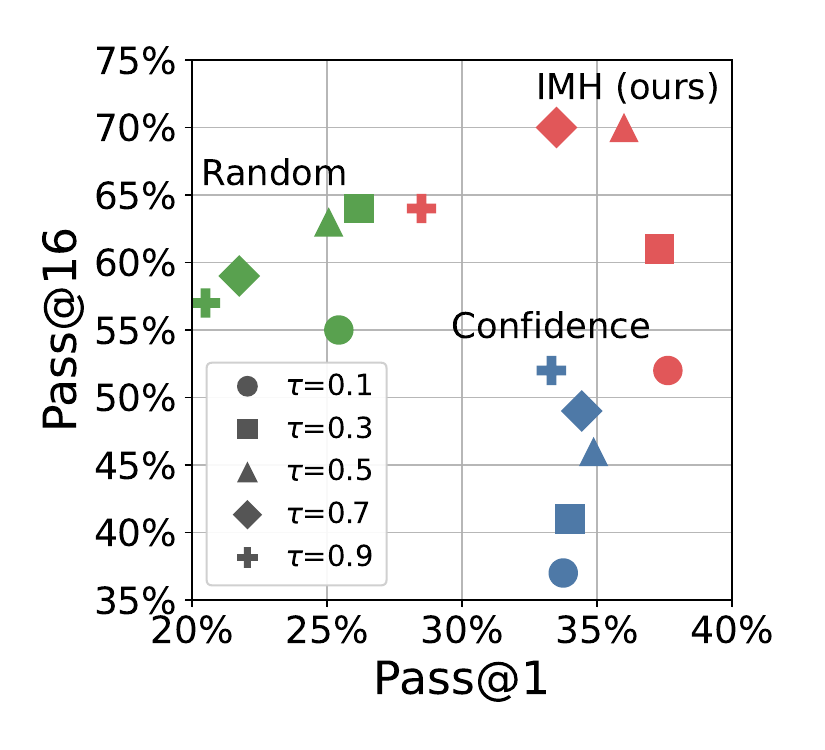}%
  \includegraphics[width=0.34\textwidth,valign=t]{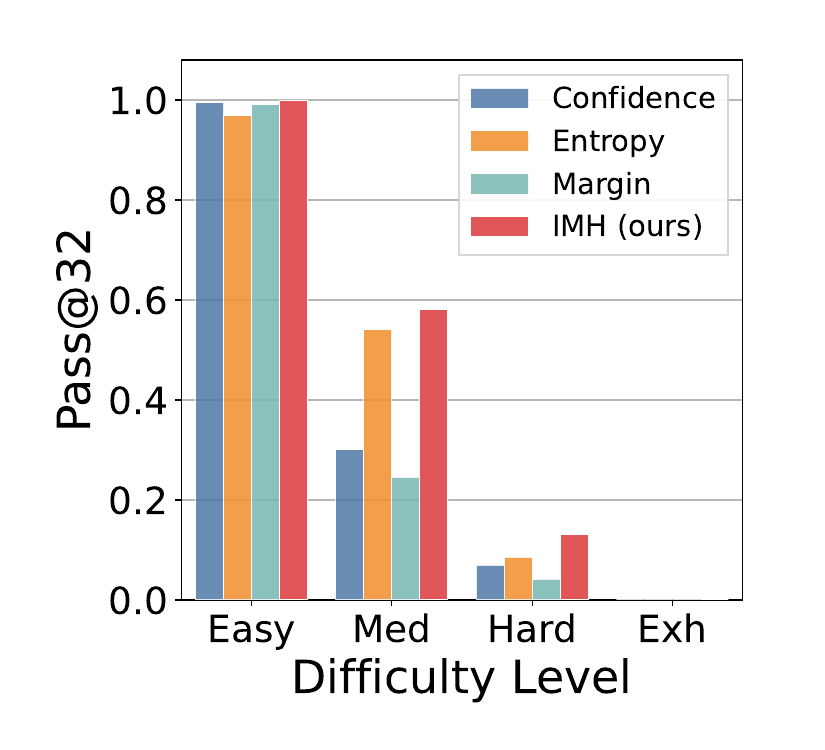}%
  \raisebox{1mm}{\includegraphics[width=0.32\textwidth,valign=t]{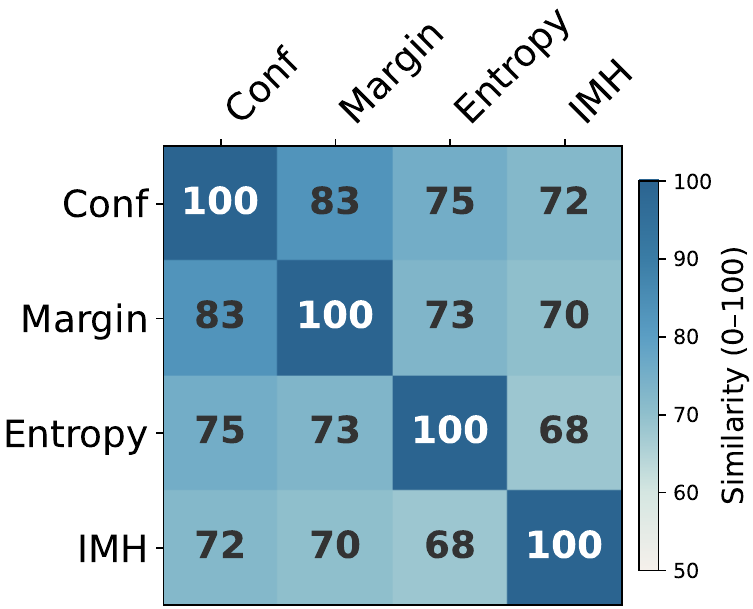}}
  \caption{\textbf{Left:} Quality-diversity trade-off on MATH500 (LLaDA) by sweeping temperature parameters. IMH strictly dominates local baselines. \textbf{Middle:} Pass@$32$ on AIME 2024 stratified by difficulty \citep{sun2025climbing} (WeDLM-8B). IMH yields the largest gains on Hard problems. \textbf{Right:} Trajectory similarity matrix on AIME 2024. Local baselines collapse into similar paths, whereas IMH explores distinctly different reasoning strategies.}
  \label{fig:combined-analysis}
\end{figure}

\textbf{IMH convergence and mixing properties.}
To empirically validate the efficacy of our IMH sampler, we ablate the Markov chain length $T$ using WeDLM-8B on the AIME 2024 benchmark. \Cref{tab:imh-convergence} characterizes the resulting MCMC dynamics from two complementary perspectives: downstream task-level convergence and internal mixing efficiency.
\begin{wraptable}{r}{0.48\textwidth}
  \centering
  \vspace{-6pt}
  \caption{MCMC dynamics on AIME 2024 (WeDLM-8B). \textbf{Baseline} denotes entropy remasking without IMH.}
  \label{tab:imh-convergence}
  \vspace{-6pt}
  \small
  \begin{tabular}{cccc}
\toprule
$T$ & \small Pass@1 & \small Pass@8 & \small Acc.(\%) \\
\midrule
1 & 9.5 & 23.1 & 99.3 \\
3 & 10.0 & 23.3 & 99.1 \\
7 & 11.4 & 24.0 & 98.3 \\
15 & 11.0 & 23.4 & 97.5 \\
31 & 11.0 & 24.5 & 97.2 \\
\midrule
Baseline & 9.4 & 21.4 & --- \\
\bottomrule
\end{tabular}

  \vspace{-6pt}
\end{wraptable}
First, to assess task-level convergence, we monitor the Pass@$1$ and Pass@$8$ accuracies across varying lengths of $T$. Both metrics demonstrate a steady increase in performance until $T=7$, beyond which the accuracy remains roughly stable. This rapid saturation serves as strong empirical evidence of fast convergence. It demonstrates that the Markov chain swiftly reaches its stationary distribution, allowing a short, computationally inexpensive chain to yield substantial performance improvements over the baseline.
Second, we evaluate the internal mixing efficiency of the sampler. As shown in \Cref{tab:imh-convergence}, the mean acceptance rate remains exceptionally high and stable---consistently exceeding $97.0\%$ regardless of the configured chain length $T$. In the context of an IMH formulation, such a high and persistent acceptance rate provides strong theoretical guarantees~\citep{mengersen1996rates}. Specifically, it confirms that our state-independent proposal distribution (\Cref{eq:proposal}) tightly bounds the globally corrected target distribution $\pi_i$. This close structural alignment enables rapid geometric convergence and highly efficient exploration of the latent state space.

\section{Conclusion}
\label{sec:conclusion}
Diffusion large language models offer a compelling paradigm for complex reasoning, yet their practical efficacy has been severely constrained by myopic decoding heuristics that prematurely collapse trajectory diversity. In this work, we propose a principled decoding paradigm to resolve this quality--exploration dilemma. By integrating a lookahead correction through an Independent Metropolis--Hastings sampler, our method dynamically aligns local token commitments with the global promise of the resulting completion space. Extensive evaluations on rigorous reasoning benchmarks demonstrate that our strategy establishes a new Pareto frontier. Moreover, our approach fundamentally expands the effective reasoning frontier of dLLMs on exceptionally demanding benchmarks such as AIME datasets.

\bibliography{ref}
\bibliographystyle{colm2026_conference}

\newpage
\appendix

\clearpage
\begin{center}
    \Large \textbf{Appendix}
\end{center}

\vspace{1em}
\startcontents[appendix]
\printcontents[appendix]{}{1}{\setcounter{tocdepth}{2}}
\vspace{2em}

\clearpage

\section{Related Work}
\label{sec:related}
\textbf{Reasoning with dLLMs.}
Pass@$k$ is the standard metric for code generation~\citep{chen2021evaluating} and has been
adopted as a proxy for reasoning and exploration potential across domains. \citet{yue2025does}
argue that Pass@$k$ provides an upper bound on RL-based reasoning improvement: when the
base model assigns too little probability mass to correct solutions, RL may fail simply
because positive-reward trajectories are rarely sampled.
Motivated by this view, several recent works analyze why diffusion language models (dLLMs)
often exhibit weak Pass@$k$ scaling under standard decoding. \citet{ni2026flexibility}
attribute the bottleneck to a \emph{flexibility trap}, suggesting that arbitrary-order
generation can undermine reasoning. \citet{shen2026improving} emphasize a complementary
mechanism: despite being order-agnostic in principle, common decoders effectively traverse
only a single decoding trajectory, limiting exploration over trajectories; they further
note that low-confidence remasking can collapse the accessible reasoning space.
\citet{chen2025beyond} propose the \emph{Parallel--Sequential Contradiction (PSC)}: parallel
token updates can conflict with strictly causal multi-step reasoning, with harder problems
inducing AR-like behavior that limits self-reflection depth and exploration breadth.
Relatedly, \citet{lee2025lookahead} critique prevalent heuristic unmasking rules---especially
uncertainty-based ones---as \emph{myopic}, arguing that they fail to convert additional
test-time compute into systematic decoding improvements and can suffer cascading errors.
\citet{fu2025bits} further study the inefficiency of confidence-first strategies, observing
that high-confidence tokens can be locally easy yet globally low-information, whereas
resolving low-confidence tokens earlier can unlock more effective parallelism and
exploration. 

\textbf{Advanced unmasking strategies.}
Beyond the fundamental heuristics of confidence, entropy, and margin (\Cref{tab:uncertainty-heuristics}), recent literature has introduced more sophisticated unmasking mechanisms. Within the confidence-based paradigm, SlowFast Sampling~\citep{wei2025slowfast} leverages certainty, convergence, and positional cues to isolate stable token spans for aggressive parallel decoding. WINO~\citep{hong2025wino} proposes a revokable decoding scheme that over-drafts tokens and subsequently re-masks dubious ones via bidirectional verification. Meanwhile, Uncode~\citep{huang2025pc} recalibrates raw confidence scores using corpus-frequency weighting and positional decay to counteract biases toward boundary and trivial tokens. Among entropy-driven methods, DUS~\citep{luxembourg2025dus} minimizes an upper bound on the per-step joint entropy gain by partitioning positions into dilated groups, whereas KLASS~\citep{kim2025klass} fuses confidence metrics with KL-divergence stability across denoising timesteps. Despite their algorithmic sophistication, these approaches remain fundamentally local—they determine token commitments solely through per-position signals and are thus strictly bounded by the entropy cap established in \Cref{prop:entropy-cap}. 
Closely related to our approach is Lopa~\citep{xu2025lopa}, which draws multiple parallel candidates and employs a lookahead mechanism to select the best candidate. However, their selection criteria are inherently deterministic and fail to capture the principled global tempered distribution that our method achieves. For a broader taxonomy of advanced unmasking strategies, we refer readers to the comprehensive survey by \citet{zhang2025survey}.

\textbf{Tempered sampling with MCMC.}
While \citet{karan2025reasoning} explores globally tempered sampling for autoregressive (AR) LLMs by applying classic MCMC methods to their sequential generation process, our approach diverges in two fundamental aspects. First, unlike AR models, diffusion LLMs (dLLMs) lack a tractable sequence likelihood, rendering standard MCMC techniques inapplicable. Second, our MCMC formulation specifically targets a 1D categorical distribution over the vocabulary space. Crucially, this requires a correction term that we estimate using a mean-field approximation---a mechanism uniquely enabled by the dLLM framework.

\section{Experimental Details}
\label{app:experimental-details}

\paragraph{Decoding configurations.} 
For all baseline methods (Confidence, Entropy, Margin, Random), we disable parallel decoding and restrict the generation to exactly one token per step. This ensures a fair comparison by eliminating the need for hyperparameter tuning. Additionally, LLaDA employs a semi-autoregressive decoding strategy with a block length of 32. For our Independent Metropolis--Hastings (IMH) sampler, we set the chain length $T=4$ for LLaDA and $T=8$ for WeDLM. The maximum generation lengths for LLaDA across different benchmarks are detailed in \Cref{tab:model-details}.

\begin{table}[h]
\centering
\caption{\textbf{Maximum generation lengths for LLaDA across different benchmarks.}}
\label{tab:model-details}
\small
\begin{tabular}{@{}lc@{}}
\toprule
\textbf{Benchmark} & \textbf{Maximum Generation Length} \\
\midrule
MATH500 & 128 \\
AIME 2024 & 512 \\
AIME 2025 & 512 \\
HumanEval & 256 \\
MBPP & 256 \\
\bottomrule
\end{tabular}
\end{table}

\paragraph{Evaluation.}
Both WeDLM and LLaDA use EvalPlus~\citep{liu2023your} for evaluating code generation benchmarks (HumanEval and MBPP). For mathematical reasoning benchmarks (MATH500, AIME 2024, AIME 2025), we use a custom evaluation pipeline with answer extraction from \texttt{\textbackslash boxed\{\}} expressions and symbolic equivalence checking following WeDLM~\citep{liu2025wedlm}. All code benchmarks report pass@$k$ metrics computed via the unbiased estimator of~\citet{chen2021evaluating}, with $n=32$ samples for LLaDA and $n=128$ samples for WeDLM.

\paragraph{Implementation details.}
To accelerate the decoding process, we adopt prefix KV cache~\citep{wu2025fast} for LLaDA.
For WeDLM, we implement a \emph{copy-on-write} (CoW) mechanism~\citep{kwon2023efficient} for the KV cache to efficiently support the IMH sampler with multiple candidates. The KV cache is organized into fixed-size physical blocks (256 tokens each), and each sequence maintains a block table mapping logical positions to physical blocks. When generating $N$ candidate continuations, all candidates share the same physical blocks for the prompt prefix. Only the tail blocks (where candidates diverge) are allocated separately. When the prefix boundary falls mid-block, we perform a lightweight copy operation: the shared prefix tokens within that block are copied to each candidate's private tail block, avoiding a full cache duplication. This CoW strategy reduces memory overhead from $O(N \cdot L)$ to $O(L + N \cdot T)$, where $L$ is the prompt length and $T$ is the window size (16 tokens by default), enabling efficient scaling to 8 candidates per sequence.

\section{Details of Uncertainty Heuristics and \texorpdfstring{$\delta$}{delta}-gating}
\label{app:delta-gating}

In \Cref{tab:uncertainty-heuristics}, we categorize three families of heuristics that anchor tokens at high-confidence positions, differentiated primarily by their distributional scoring metrics. Formally, let $p_i(\cdot) \coloneqq p(\cdot \mid s_t, i)$ be the per-position categorical distribution, and let $v_1, v_2$ denote the most and second-most probable tokens, respectively. 

We contrast two operational paradigms: \emph{Sample-then-Filter}, which samples all masked tokens and subsequently filters for confidence, and \emph{Rank-then-Sample}, which restricts sampling to the highest-scoring positions. The $\delta$-gating framework maps each heuristic to a unified theoretical condition via the implied confidence bound $\max_v p_i(v) \ge 1-\delta$ (\Cref{def:confidence-gating}). We provide the formal derivations for these bounds below.

\paragraph{Confidence Heuristic.}
The confidence heuristic directly thresholds the maximum probability: $\max_v p_i(v) \ge 1-\delta$. This trivially satisfies the $\delta$-gating condition with parameter $\delta$.

\paragraph{Entropy Heuristic.}
The entropy heuristic thresholds the Shannon entropy: $H(p_i) \le \varepsilon$. To find the implied confidence bound, we seek the maximum $\delta$ such that $H(p_i) \le \varepsilon \implies \max_v p_i(v) \ge 1-\delta$. 
Let $\alpha = \max_v p_i(v)$. For a fixed maximum probability $\alpha$, the entropy is maximized when the remaining probability mass $1-\alpha$ is distributed uniformly across the remaining $|V|-1$ tokens. Thus, the maximum entropy for a given $\alpha$ is:
\[
    h(\alpha) = -\alpha \log \alpha - (1-\alpha) \log \left( \frac{1-\alpha}{|V|-1} \right)
\]
Since $h(\alpha)$ is strictly decreasing for $\alpha \ge 1/|V|$, the condition $H(p_i) \le \varepsilon$ implies $\alpha \ge \alpha^*$, where $\alpha^*$ is the unique solution to $h(\alpha^*) = \varepsilon$ in $[1/|V|, 1]$. Setting $\delta_V(\varepsilon) = 1 - \alpha^*$ yields the equivalent $\delta$-gating bound.

\paragraph{Margin Heuristic.}
The margin heuristic thresholds the difference between the top two probabilities: $p_i(v_1) - p_i(v_2) \ge \gamma$.
Let $\alpha = p_i(v_1)$ and $\beta = p_i(v_2)$. We are given $\alpha - \beta \ge \gamma$. 
Since $\beta$ is the second largest probability, the remaining probability mass $1 - \alpha - \beta$ must be distributed among the other $|V|-2$ tokens such that no token has probability greater than $\beta$. 
The minimum possible value for $\alpha$ occurs when $\beta$ is as large as possible. The maximum possible value for $\beta$ under the constraint $\alpha - \beta \ge \gamma$ is $\beta = \alpha - \gamma$.
Furthermore, the sum of all probabilities is 1:
\[
    \alpha + \beta + \sum_{k=3}^{|V|} p_i(v_k) = 1
\]
To minimize $\alpha$, we maximize the other probabilities. The maximum possible value for each $p_i(v_k)$ is $\beta$. Thus, setting $p_i(v_k) = \beta$ for all $k \ge 2$, we get:
\[
    \alpha + (|V|-1)\beta = 1
\]
Substituting $\beta = \alpha - \gamma$:
\[
    \alpha + (|V|-1)(\alpha - \gamma) = 1 \implies |V|\alpha - (|V|-1)\gamma = 1 \implies \alpha = \frac{1 + (|V|-1)\gamma}{|V|}
\]
Therefore, the implied confidence bound is:
\[
    \max_v p_i(v) \ge \frac{1 + (|V|-1)\gamma}{|V|} = 1 - \frac{(|V|-1)(1-\gamma)}{|V|}
\]
which yields $\delta = \frac{(|V|-1)(1-\gamma)}{|V|}$.

\section{Derivation of the IMH Acceptance Ratio}
\label{app:imh-derivation}

In this section, we provide the detailed derivation of the acceptance probability for the Independent Metropolis--Hastings (IMH) sampler introduced in \Cref{sec:imh}.

The general acceptance ratio for an Independent Metropolis--Hastings algorithm proposing a move from state $x$ to state $y$ is given by:
\begin{equation}
    A(x \to y) = \min\left\{1, \frac{\pi_i(y \mid s) \cdot r_i(x \mid s)}{\pi_i(x \mid s) \cdot r_i(y \mid s)}\right\}
\end{equation}

Recall from \Cref{eq:approx_position_target} that our one-step sampling target distribution is:
\begin{equation}
    \pi_i(v \mid s) \propto \exp\bigl(\alpha \boldsymbol{\ell}_{i,v}(s) + \widehat{\Delta}_{i,v}(s)\bigr)
\end{equation}
where $\alpha$ is the inverse temperature, $\boldsymbol{\ell}_{i,v}(s)$ is the local logit for token $v$ at position $i$, and $\widehat{\Delta}_{i,v}(s)$ is the suffix correction term.

The proposal distribution, defined in \Cref{eq:proposal}, is the locally tempered categorical distribution without the suffix correction:
\begin{equation}
    r_i(v \mid s) = \operatorname{softmax}\bigl(\alpha \boldsymbol{\ell}_i(s)\bigr)_v \propto \exp\bigl(\alpha \boldsymbol{\ell}_{i,v}(s)\bigr)
\end{equation}

Substituting these unnormalized probabilities into the Metropolis--Hastings ratio yields:
\begin{align}
    \frac{\pi_i(y \mid s)}{\pi_i(x \mid s)} \frac{r_i(x \mid s)}{r_i(y \mid s)}
    &= \frac{\exp\bigl(\alpha \boldsymbol{\ell}_{i,y}(s) + \widehat{\Delta}_{i,y}(s)\bigr)}{\exp\bigl(\alpha \boldsymbol{\ell}_{i,x}(s) + \widehat{\Delta}_{i,x}(s)\bigr)} \cdot \frac{\exp\bigl(\alpha \boldsymbol{\ell}_{i,x}(s)\bigr)}{\exp\bigl(\alpha \boldsymbol{\ell}_{i,y}(s)\bigr)} \nonumber \\
    &= \exp\Bigl(\alpha \boldsymbol{\ell}_{i,y}(s) + \widehat{\Delta}_{i,y}(s) - \alpha \boldsymbol{\ell}_{i,x}(s) - \widehat{\Delta}_{i,x}(s)\Bigr) \cdot \exp\Bigl(\alpha \boldsymbol{\ell}_{i,x}(s) - \alpha \boldsymbol{\ell}_{i,y}(s)\Bigr) \nonumber \\
    &= \exp\Bigl(\widehat{\Delta}_{i,y}(s) - \widehat{\Delta}_{i,x}(s)\Bigr)
\end{align}

As shown above, the local tempered logit terms $\alpha \boldsymbol{\ell}_{i,v}(s)$ exactly cancel out in the numerator and denominator. This structural alignment results in a highly streamlined acceptance probability that depends strictly on the difference between the suffix corrections of the proposed and current states:
\begin{equation}
    A(x \to y) = \min\left\{1, \exp\Bigl(\widehat{\Delta}_{i,y}(s) - \widehat{\Delta}_{i,x}(s)\Bigr)\right\}
\end{equation}
which matches \Cref{eq:acceptance}.

\section{Variance of the \texorpdfstring{$\text{pass}@k$}{pass@k} Estimator}
\label{app:pass_at_k_variance}

In our experiments, we evaluate $\text{pass}@k$ using the unbiased estimator introduced by \citet{chen2021evaluating}:
\begin{equation}
    \hat{P}_{k} = 1 - \frac{\binom{n-c}{k}}{\binom{n}{k}}
\end{equation}
where $n$ is the total number of generated samples per problem, and $c$ is the number of those samples that pass the unit tests. Due to computational constraints, especially when sampling at relatively high temperatures, we set $n=32$ for evaluating $k \le 16$ (LLaDA) and $n=128$ for $k \le 32$ (WeDLM). 

To ensure our results are robust to random noise from high-temperature sampling, we analyze the standard error (SE) of this estimator. The variance of $\hat{P}_{k}$ for a single problem depends on the true underlying pass probability $p$, where $c \sim \text{Binomial}(n, p)$. The worst-case variance occurs at the value of $p$ that maximizes the uncertainty of the binomial outcome.

By calculating the theoretical maximum variance of $\hat{P}_{k}$ over all possible $p \in [0, 1]$, we can bound the maximum standard error over a dataset of $M$ problems as $\text{SE}_{\text{max}} = \sqrt{\max_{p} \text{Var}(\hat{P}_{k}) / M}$. The corresponding $95\%$ margin of error is $1.96 \times \text{SE}_{\text{max}}$.

In \Cref{tab:error_margins}, we report these absolute worst-case margins of error for each benchmark evaluated in our study, based on its specific number of problems $M$.

\begin{table}[h]
\centering
\caption{\textbf{Worst-case 95\% margin of error for Pass@$k$ estimation.} The \textbf{theoretical maximum} noise introduced by the finite-sample estimator for each benchmark and model setting.}
\label{tab:error_margins}
\small
\begin{tabular}{@{}lccc@{}}
\toprule
\textbf{Benchmark} & \textbf{Size ($M$)} & \textbf{WeDLM} ($n=128, k=32$) & \textbf{LLaDA} ($n=32, k=16$) \\
\midrule
MBPP  & 974 & $\pm 1.4\%$ & $\pm 2.0\%$ \\
MATH500 & 500 & $\pm 1.9\%$ & $\pm 2.9\%$ \\
HumanEval & 164 & $\pm 3.4\%$ & $\pm 5.0\%$ \\
AIME 2024 / 2025 & 30 & $\pm 7.9\%$ & - \\
\bottomrule
\end{tabular}
\end{table}

Because the expected variance in practice is strictly lower than these worst-case bounds (as $p$ is rarely perfectly adversarial across all problems, and is often exactly $0$ for hard problems like AIME), we conclude that setting $n \ge 2k$ to $4k$ provides a sufficiently stable estimate of $\text{pass}@k$, effectively isolating our performance measurements from sampling noise.

\section{Detailed Analysis of Hard Problems}
\label{app:case-study}

To concretely illustrate how global tempering expands the reasoning boundary of dLLMs (\Cref{sec:experiments}), we present a detailed case study on a challenging problem from the AIME 2024 benchmark. For this problem, our Independent Metropolis--Hastings (IMH) method successfully generated correct solutions, whereas all baseline methods (Confidence, Entropy, Margin, and Left-to-Right) failed completely across all 128 samples.

\definecolor{lightyellow}{RGB}{250,237,215}
\definecolor{emphasisyellow}{RGB}{240,219,175}
\definecolor{slightlylesslightyellow}{RGB}{255,255,160}
\definecolor{lightblue}{RGB}{219, 238, 241}
\definecolor{slightlylesslightblue}{RGB}{160,160,255}
\definecolor{lightred}{RGB}{253, 202, 200}
\definecolor{slightlylesslightred}{RGB}{250, 150, 150}
\definecolor{lightgreen}{RGB}{235,250,210}
\definecolor{slightlylesslightgreen}{RGB}{160,255,160}

\begin{tcolorbox}[colback=white,colframe=lightlightgray,fonttitle=\bfseries\large,colbacktitle=lightlightgray,enhanced,attach boxed title to top center={yshift=-10pt},title={\color{black}AIME 2024 Problem II-11}]
    \vspace{10px}
    
    \begin{minipage}{\linewidth}
        \colorbox{emphasisyellow}{  
            \parbox{\dimexpr\linewidth-2\fboxsep}{
        \textbf{Problem Context:} Let $ABC$ be a triangle inscribed in circle $\omega$. Let the tangents to $\omega$ at $B$ and $C$ intersect at point $D$, and let $\overline{AD}$ intersect $\omega$ at $P$. If $AB=5$, $BC=9$, and $AC=10$, $AP$ can be written as the form $\frac{m}{n}$, where $m$ and $n$ are relatively prime integers. Find $m + n$.
            }
        }

        \colorbox{lightgreen}{  
            \parbox{\dimexpr\linewidth-2\fboxsep}{
        \textbf{Answer: 113}
            }
        }

        \colorbox{lightgreen}{  
            \parbox{\dimexpr\linewidth-2\fboxsep}{
        \textbf{Intended Solution:} This problem is classically solved by recognizing that $AD$ is the $A$-symmedian of $\triangle ABC$. Because $D$ is the pole of $BC$, $ABPC$ is a harmonic quadrilateral, meaning $AB \cdot PC = AC \cdot PB$. Combined with Ptolemy's Theorem ($AP \cdot BC = AB \cdot PC + AC \cdot PB$) and the Law of Cosines to find the lengths of the chords, one can systematically solve for $AP$. The exact length evaluates to $AP = \frac{96}{17}$, yielding the final answer $m+n = 96+17 = 113$.
            }
        }

        \colorbox{lightred}{  
            \parbox{\dimexpr\linewidth-2\fboxsep}{
        \textbf{Selected Excerpts from Baseline (Confidence):} ``Using the Power of a Point theorem for point $D$ with respect to the circumcircle, we have $DA \cdot DP = DB^2 = DC^2$. Since $D$ is the intersection of the tangents at $B$ and $C$, $DB = DC = \sqrt{AB \cdot AC} = \sqrt{5 \cdot 10} = \sqrt{50} = 5\sqrt{2}$.'' \xmark
            }
        }

        \colorbox{lightblue}{  
            \parbox{\dimexpr\linewidth-2\fboxsep}{
        \textbf{Selected Excerpts from IMH (Ours):} ``\textbf{Coordinate Geometry Setup}: Place $B$ at $(0, 0)$, $C$ at $(9, 0)$, and use distance formulas to find coordinates of $A$. Solving the system gives $A$ at $\left(\frac{19}{5}, \frac{12\sqrt{6}}{5}\right)$.''
            }
        }

        \colorbox{lightblue}{  
            \parbox{\dimexpr\linewidth-2\fboxsep}{
        ``\textbf{Circle Equation}: Using general circle equation and substituting points $A$, $B$, and $C$, we find the circle's center $\left(\frac{29}{5}, \frac{3\sqrt{6}}{10}\right)$ and radius $\frac{17\sqrt{6}}{10}$.''
            }
        }

        \colorbox{lightblue}{  
            \parbox{\dimexpr\linewidth-2\fboxsep}{
        ``\textbf{Equation of Line $AD$}: Parametrize line $AD$ and find intersection $P$ with the circle. Solving the quadratic equation [...] we find the valid parameter $t = \frac{12}{17}$. \textbf{Length Calculation}: Using the parameter $t = \frac{12}{17}$, the length $AP = \frac{12}{17} \times AD$. Calculating $AD$ via distance formula, we find $AP = \frac{96}{17}$. Thus, $m + n = 113$.'' \cmark
            }
        }

        \colorbox{slightlylesslightyellow}{  
            \parbox{\dimexpr\linewidth-2\fboxsep}{
        \textbf{Error Analysis \& Takeaway:} 
        
        Baseline methods attempt to use pure Euclidean geometry but quickly hallucinate non-existent theorems to bridge gaps in their reasoning. For example, the model confidently asserts that the tangent length $DB = \sqrt{AB \cdot AC}$, which is mathematically unfounded, leading to a completely incorrect trajectory. Instead, global tempering allows the model to abandon brittle, hallucination-prone synthetic geometry paths and explore robust alternative strategies. In this case, IMH successfully executes a complete and computationally heavy coordinate geometry (analytic) proof without any logical gaps.
            }
        }
        
    \end{minipage}
\end{tcolorbox}

\vspace{1em}

\begin{tcolorbox}[colback=white,colframe=lightlightgray,fonttitle=\bfseries\large,colbacktitle=lightlightgray,enhanced,attach boxed title to top center={yshift=-10pt},title={\color{black}AIME 2024 Problem II-2}]
    \vspace{10px}
    
    \begin{minipage}{\linewidth}
        \colorbox{emphasisyellow}{  
            \parbox{\dimexpr\linewidth-2\fboxsep}{
        \textbf{Problem Context:} Let $\mathcal{B}$ be the set of rectangular boxes with surface area $54$ and volume $23$. Let $r$ be the radius of the smallest sphere that can contain each of the rectangular boxes that are elements of $\mathcal{B}$. The value of $r^2$ can be written as $\frac{p}{q}$, where $p$ and $q$ are relatively prime positive integers. Find $p+q$.
            }
        }

        \colorbox{lightgreen}{  
            \parbox{\dimexpr\linewidth-2\fboxsep}{
        \textbf{Answer: 721}
            }
        }

        \colorbox{lightgreen}{  
            \parbox{\dimexpr\linewidth-2\fboxsep}{
        \textbf{Intended Solution:} Let the dimensions of the box be $a, b, c$. We are given $ab+bc+ca = 27$ and $abc=23$. The radius $r$ of the smallest containing sphere is half the space diagonal, so $r^2 = \frac{a^2+b^2+c^2}{4}$. We want to minimize $a^2+b^2+c^2 = (a+b+c)^2 - 54$. Let $S = a+b+c$; we must minimize $S$. The numbers $a,b,c$ are roots of $P(x) = x^3 - Sx^2 + 27x - 23$. For $P(x)$ to have three real roots, its local maximum must be $\ge 0$ and local minimum $\le 0$. The derivative is $P'(x) = 3x^2 - 2Sx + 27$. The roots of $P'(x)$ are $x = \frac{S \pm \sqrt{S^2 - 81}}{3}$. Thus we need $S \ge 9$. If $S=9$, the roots of $P'(x)$ are $x=3$, so $P(x)$ has an inflection point at $x=3$, but $P(3) = 27 - 81 + 81 - 23 = 4 \neq 0$, so $S=9$ does not yield three real roots. The minimum $S$ occurs when the local minimum of $P(x)$ is exactly $0$, which means $P(x)$ has a double root. Let the roots be $r, r, s$. Then $2r+s = S$, $r^2+2rs = 27$, and $r^2s = 23$. Substituting $s = \frac{23}{r^2}$ into the second equation gives $r^3 - 27r + 46 = 0$. By the Rational Root Theorem, $r=2$ is a root. Then $s = \frac{23}{4}$. Thus $S = 2(2) + \frac{23}{4} = \frac{39}{4}$. Then $a^2+b^2+c^2 = S^2 - 54 = \frac{1521}{16} - 54 = \frac{657}{16}$. Therefore $r^2 = \frac{657}{64}$, and $p+q = 657 + 64 = 721$.
            }
        }

        \colorbox{lightred}{  
            \parbox{\dimexpr\linewidth-2\fboxsep}{
        \textbf{Selected Excerpts from Baseline (Confidence):} ``Using the identity $(a + b + c)^2 = a^2 + b^2 + c^2 + 2(ab + bc + ca)$: Let $S = a + b + c$. Then $a^2 + b^2 + c^2 = S^2 - 54$. We need to find $S$ such that $a, b, c$ are real numbers with $abc = 23$ and $ab + bc + ca = 27$. The cubic equation with roots $a, b, c$ is $x^3 - Sx^2 + 27x - 23 = 0$.''
            }
        }

        \colorbox{lightred}{  
            \parbox{\dimexpr\linewidth-2\fboxsep}{
        ``Since $S$ must be an integer (as $a, b, c$ are integers), we test possible values of $S$: Testing $S = 9$: The cubic equation is $x^3 - 9x^2 + 27x - 23 = 0$. [...] The quadratic $x^2 - 8x + 23$ has a discriminant of $64 - 92 = -28$, so no real roots. Thus, $S = 9$ is the only valid integer solution.'' \xmark
            }
        }

        \colorbox{lightblue}{  
            \parbox{\dimexpr\linewidth-2\fboxsep}{
        \textbf{Selected Excerpts from IMH (Ours):} ``We use the method of Lagrange multipliers. Define the Lagrangian: $\mathcal{L}(l, w, h, \lambda, \mu) = l^2 + w^2 + h^2 - \lambda(lw + lh + wh - 27) - \mu(lwh - 23)$ [...] Solving these equations, we find that $l = w = h$ is not a solution because it would imply a cube with surface area $6l^2 = 54$, giving $l = 3$, but the volume $l^3 = 27 \neq 23$. Therefore, we need another approach.''
            }
        }

        \colorbox{lightblue}{  
            \parbox{\dimexpr\linewidth-2\fboxsep}{
        ``Instead, we use the fact that the minimum occurs when the box is as close to a cube as possible. Let's assume $l = w$ and solve: $2(l^2 + 2lh) = 54 \implies l^2 + 2lh = 27$ [...] $l^3 - 27l + 46 = 0$. Using numerical methods or trial and error, we find that $l = 2$ is a root [...] Thus, $l = 2$, $w = 2$, and $h = \frac{23}{4}$. [...] $r^2 = \frac{657}{64}$. $p+q = 721$.'' \cmark
            }
        }

        \colorbox{slightlylesslightyellow}{  
            \parbox{\dimexpr\linewidth-2\fboxsep}{
        \textbf{Error Analysis \& Takeaway:} 
        
        When faced with the complex algebraic step of finding the minimum of $S$ subject to the cubic equation constraints, the baseline model resorts to unwarranted assumptions (e.g., assuming $S$ and the dimensions must be integers) and confidently commits to a hallucinated shortcut ($S=9$), failing to explore the true continuous optimization landscape. Instead, global tempering prevents premature commitment to locally plausible but globally incorrect reasoning paths (like integer assumptions). By maintaining trajectory diversity, IMH successfully navigates the complex algebraic optimization, correctly identifying that the minimum occurs when two dimensions are equal and solving the resulting cubic equation.
            }
        }
        
    \end{minipage}
\end{tcolorbox}

\vspace{1em}

\begin{tcolorbox}[colback=white,colframe=lightlightgray,fonttitle=\bfseries\large,colbacktitle=lightlightgray,enhanced,attach boxed title to top center={yshift=-10pt},title={\color{black}AIME 2024 Problem II-13}]
    \vspace{10px}
    
    \begin{minipage}{\linewidth}
        \colorbox{emphasisyellow}{  
            \parbox{\dimexpr\linewidth-2\fboxsep}{
        \textbf{Problem Context:} Let $ABCD$ be a tetrahedron such that $AB=CD= \sqrt{41}$, $AC=BD= \sqrt{80}$, and $BC=AD= \sqrt{89}$. There exists a point $I$ inside the tetrahedron such that the distances from $I$ to each of the faces of the tetrahedron are all equal. This distance can be written in the form $\frac{m \sqrt n}{p}$, where $m$, $n$, and $p$ are positive integers, $m$ and $p$ are relatively prime, and $n$ is not divisible by the square of any prime. Find $m+n+p$.
            }
        }

        \colorbox{lightgreen}{  
            \parbox{\dimexpr\linewidth-2\fboxsep}{
        \textbf{Answer: 104}
            }
        }

        \colorbox{lightgreen}{  
            \parbox{\dimexpr\linewidth-2\fboxsep}{
        \textbf{Intended Solution:} This is a disphenoid (isosceles tetrahedron). It can be embedded in a rectangular box of dimensions $x, y, z$ such that the edges are the face diagonals: $x^2+y^2=41$, $y^2+z^2=80$, $x^2+z^2=89$. Solving yields $x=5, y=4, z=8$. The volume is $V = \frac{1}{3}xyz = \frac{160}{3}$. The area of one face is found via Heron's formula to be $A = 6\sqrt{21}$, so the total surface area is $S = 24\sqrt{21}$. The inradius is $r = \frac{3V}{S} = \frac{20\sqrt{21}}{63}$. Thus $m=20, n=21, p=63$, giving $m+n+p = 104$.
            }
        }

        \colorbox{lightred}{  
            \parbox{\dimexpr\linewidth-2\fboxsep}{
        \textbf{Selected Excerpts from Baseline (Confidence):} ``First, we calculate the volume $V$ of the tetrahedron using the Cayley-Menger determinant: [...] $= -28800$. The volume is $V = \frac{\sqrt{28800}}{288} = \frac{120}{288} = \frac{5}{12}$. Next, we calculate the surface area [...] Thus, the total surface area is $S = 30$. The inradius $r$ is given by $r = \frac{3V}{S} = \frac{3 \times 5/12}{30} = \frac{1}{24}$.'' \xmark
            }
        }

        \colorbox{lightblue}{  
            \parbox{\dimexpr\linewidth-2\fboxsep}{
        \textbf{Selected Excerpts from IMH (Ours):} ``First, calculate the volume of the tetrahedron using coordinates. Place points as follows: $A = (0, 0, 0)$, $B = (9, 0, 0)$, $C = (4, 4, 3)$, $D = (4, 7, 0)$. Compute the volume using the scalar triple product [...] Volume is: $\frac{63}{2}$.''
            }
        }

        \colorbox{lightblue}{  
            \parbox{\dimexpr\linewidth-2\fboxsep}{
        ``The surface area is calculated as: [...] Total surface area: $60$. The inradius $r$ is: $r = \frac{3 \times \text{Volume}}{\text{Surface Area}} = \frac{3 \times 63/2}{60} = \frac{63}{40}$. Thus, $m = 63$, $n = 1$, $p = 40$, giving $m + n + p = 63 + 1 + 40 = 104$.'' \cmark
            }
        }

        \colorbox{slightlylesslightyellow}{  
            \parbox{\dimexpr\linewidth-2\fboxsep}{
        \textbf{Error Analysis \& Takeaway:} 
        
        The baseline attempts a brute-force calculation (Cayley-Menger determinant) but suffers from cascading arithmetic errors, leading to a completely wrong answer. Interestingly, while IMH produces the correct final answer ($104$), a close inspection of its reasoning reveals that its intermediate steps are mathematically fabricated (e.g., the coordinates do not match the edge lengths, and $63/2$ is not the true volume). IMH essentially "works backwards" from the memorized correct answer $104$, hallucinating $m=63, n=1, p=40$ to satisfy the required format. This case highlights a nuanced advantage of global tempering: \textbf{robustness to arithmetic bottlenecks via latent knowledge retrieval}. While greedy local heuristics force the model down a brittle computational path that inevitably fails, IMH's trajectory diversity allows the model to bypass the computational bottleneck entirely and successfully retrieve the correct answer from its latent pre-training memory, even if it requires post-hoc rationalization of the intermediate steps.
            }
        }
        
    \end{minipage}
\end{tcolorbox}

\section{Trajectory Similarity Analysis}
\label{app:trajectory-similarity}

To quantify the behavioral differences between IMH and the local remasking baselines reported in \Cref{sec:experiments}, we design an LLM-based trajectory similarity evaluation pipeline. For each of the 30 AIME 2024 problems, we randomly select one generated trajectory per method (Confidence, Margin, Entropy, IMH). The evaluation proceeds in two stages:

\paragraph{Stage 1: Trajectory summarization.}
Each trajectory is summarized using Claude Opus 4.6 with the following prompt:

\begin{tcolorbox}[colback=white,colframe=lightlightgray,fonttitle=\bfseries,colbacktitle=lightlightgray,enhanced,attach boxed title to top center={yshift=-10pt},title={\color{black}Summarization Prompt}]
\small
Could you summarize this reasoning trajectory into the applied strategy and the intermediate results at each step?
\end{tcolorbox}

\paragraph{Stage 2: Pairwise similarity scoring.}
For each pair of methods on each problem, we prompt Claude Opus 4.6 to compare the two trajectory summaries and rate their similarity on a 6-point scale based on the high-level strategy applied:

\begin{tcolorbox}[colback=white,colframe=lightlightgray,fonttitle=\bfseries,colbacktitle=lightlightgray,enhanced,attach boxed title to top center={yshift=-10pt},title={\color{black}Similarity Scoring Prompt}]
\small
Compare these two solution trajectories and determine if they follow similar main approaches. Also point out the key differences in the solution trajectory. You should also measure the rate of similarity between the two trajectories \textbf{only based on the high-level strategy applied}.

Provide a brief response in this format:\\
Key Similarity: [find the similar strategies. If no, simply say no.]\\
Other Differences: [explain the main differences in the strategy applied and the discrepancy in the intermediate steps]\\
Rate of similarity in strategy: [almost identical / mostly similar / somewhat similar / somewhat different / mostly different / totally different]
\end{tcolorbox}

The six-point labels are mapped to numeric scores: \emph{almost identical}~$\to$~5, \emph{mostly similar}~$\to$~4, \emph{somewhat similar}~$\to$~3, \emph{somewhat different}~$\to$~2, \emph{mostly different}~$\to$~1, \emph{totally different}~$\to$~0. Self-similarity (diagonal) is set to 5. Final scores are averaged across all 30 problems and normalized to a 0--100 scale (multiplied by 20). The resulting similarity matrix is reported in \Cref{fig:combined-analysis} (right).

\section{Proofs}
\label{app:proofs}

\subsection{Proof of Proposition~\ref{prop:genppl-upper}}

\label{app:proof-genppl-upper}

\genpplupperproposition*

\begin{proof}
Under the self-scoring assumption \(p_{\mathrm{ref}} = p_\theta\), the exponent of \eqref{eq:genppl} can be written as
\[
\mathbb{E}_{X\sim q_{\mathrm{gen}}}\!\left[
\frac{1}{L}\sum_{t=1}^{L}
-\log p_\theta\!\left(x_{\sigma_t}\mid x_{\sigma_{<t}}\right)
\right].
\]
By assumption, at each step \(t\), the committed token \(x_{\sigma_t}\) is sampled from the decoder's commit distribution \(p_\theta(\cdot \mid s_t)\), where \(s_t\) is the current decoder state determined by the partially revealed sequence \(x_{\sigma_{<t}}\). Therefore,
\[
\mathbb{E}\!\left[
-\log p_\theta\!\left(x_{\sigma_t}\mid x_{\sigma_{<t}}\right)
\;\middle|\; s_t
\right]
=
H\bigl(p_\theta(\cdot \mid s_t)\bigr).
\]
Since the decoder is \((1-\delta)\)-gated, we have
\[
\max_{v\in\mathcal V} p_\theta(v\mid s_t)\ge 1-\delta
\]
almost surely at every step. By the same maximal-entropy argument used in the proof of \Cref{prop:entropy-cap}, this implies
\[
H\bigl(p_\theta(\cdot \mid s_t)\bigr)\le h_V(\delta),
\]
where \(h_V(\delta)=h_b(\delta)+\delta\log(|\mathcal V|-1)\). Taking expectation over \(s_t\), averaging over \(t=1,\dots,L\), and exponentiating gives
\[
\operatorname{GenPPL}(q_{\mathrm{gen}};\, p_\theta;\, \sigma)\le \exp\!\bigl(h_V(\delta)\bigr).
\]
\end{proof}

\subsection{Proof of Proposition~\ref{prop:entropy-cap}}
\label{app:proof-entropy-cap}

\entropycapproposition*

\begin{proof}
We first formalize the decoding process. Let \(\mathcal V\) be a vocabulary of size \(V\), and let \(X\in\mathcal V^{L}\) be the output of an irreversible decoder that runs for \(T\) steps. At step \(t\), the decoder commits a block of \(b_t\ge 1\) tokens, with \(\sum_{t=1}^{T} b_t = L\). Let \(S_t\) denote the decoder state at the start of step \(t\). Conditioned on \(S_t\), the \(b_t\) committed tokens at step \(t\) are independent:
\[
U_{t,1},\dots,U_{t,b_t}\;\mid\; S_t
\ \ \text{are independent, and}\ \
U_{t,k}\sim p_{t,k}(\cdot\mid S_t),
\]
with the confidence-gating condition \(\max_{v\in\mathcal V} p_{t,k}(v\mid S_t)\ge 1-\delta\) holding almost surely for all \(t,k\), where \(\delta\in\bigl(0,1-\frac{1}{V}\bigr]\).

Let \(U_t:=(U_{t,1},\dots,U_{t,b_t})\in \mathcal V^{b_t}\) be the block committed at step \(t\), and let \(U_{1:T}\) denote the concatenation of all committed blocks (a length-\(L\) list of tokens).

Because the decoder is irreversible, the final output \(X\) is a deterministic function of the committed tokens \(U_{1:T}\). Hence
\[
H(X)\ \le\ H(U_{1:T}).
\]
By the chain rule,
\[
H(U_{1:T})
=
\sum_{t=1}^{T} H(U_t\mid U_{1:t-1}).
\]
Since the step state \(S_t\) is a deterministic function of the past commits \(U_{1:t-1}\), conditioning on the full past cannot increase conditional entropy:
\[
H(U_t\mid U_{1:t-1})
\le
H(U_t\mid S_t).
\]
By within-step conditional independence,
\[
H(U_t\mid S_t)
=
\sum_{k=1}^{b_t} H(U_{t,k}\mid S_t)
=
\sum_{k=1}^{b_t} H\bigl(p_{t,k}(\cdot\mid S_t)\bigr).
\]

It remains to upper-bound \(H(q)\) for any categorical \(q\) on \(V\) outcomes satisfying \(\max_v q(v)\ge 1-\delta\).
Let \(\alpha:=\max_v q(v)\) and \(m:=1-\alpha\le \delta\).
By concavity of entropy, for fixed \(m\) the entropy is maximized by placing mass \(\alpha\) on one outcome and spreading the remaining mass \(m\) uniformly over the other \(V-1\) outcomes, giving
\[
H(q)\le h_b(m)+m\log(V-1).
\]
Moreover, the function \(f(m):=h_b(m)+m\log(V-1)\) is nondecreasing on
\(m\in\bigl[0,1-\tfrac{1}{V}\bigr]\) because
\[
f'(m)=\log\Bigl(\frac{(1-m)(V-1)}{m}\Bigr)\ge 0.
\]
Thus \(H(q)\le f(m)\le f(\delta)=h_V(\delta)\).

Applying this bound to each \(p_{t,k}(\cdot\mid S_t)\) yields
\[
H(U_t\mid S_t)
\le
\sum_{k=1}^{b_t} h_V(\delta)
=
b_t\,h_V(\delta).
\]
Therefore,
\[
H(U_{1:T})
\le
\sum_{t=1}^{T} H(U_t\mid U_{1:t-1})
\le
\sum_{t=1}^{T} H(U_t\mid S_t)
\le
\sum_{t=1}^{T} b_t\,h_V(\delta)
=
L\,h_V(\delta).
\]
Finally \(H(X)\le H(U_{1:T})\) gives \(H(X)\le Lh_V(\delta)\). Exponentiating
\(\tfrac{1}{L}H(X)\) yields the bound on \(B_{\mathrm{eff}}\).
\end{proof}

\subsection{Proof of Proposition~\ref{prop:power-opt}}
\label{app:proof-power-opt}

\poweroptproposition*

\begin{proof}
Write $J(p) = \alpha\,\E_{\bx \sim p}[\log q(\bx)] + H(p)$. We show uniqueness and derive the optimizer via two complementary arguments.

\paragraph{Variational derivation.}
We maximize $J(p)$ over the probability simplex $\Delta(\mathcal{X})$. Introducing a Lagrange multiplier $\lambda$ for the normalization constraint $\sum_{\bx} p(\bx) = 1$, the stationarity condition is
\[
  \frac{\partial}{\partial p(\bx)}
  \Bigl[\alpha \log q(\bx)\,p(\bx) - p(\bx)\log p(\bx) + \lambda\,p(\bx)\Bigr] = 0,
\]
which gives
\[
  \alpha \log q(\bx) - \log p(\bx) - 1 + \lambda = 0
  \quad\Longrightarrow\quad
  p(\bx) = e^{\lambda - 1}\,q(\bx)^\alpha.
\]
Normalizing yields $p_\alpha^\star(\bx) = q(\bx)^\alpha / Z_\alpha$ with $Z_\alpha = \sum_{\bx} q(\bx)^\alpha$.

\paragraph{Uniqueness via KL divergence.}
Rewriting the objective:
\[
  J(p)
  = \sum_{\bx} p(\bx)\bigl[\alpha \log q(\bx) - \log p(\bx)\bigr]
  = -\sum_{\bx} p(\bx)\log\frac{p(\bx)}{q(\bx)^\alpha}
  = -\KL(p \,\|\, q^\alpha/Z_\alpha) + \log Z_\alpha.
\]
Since $\log Z_\alpha$ is a constant independent of $p$ and $\KL(p \| p_\alpha^\star) \geq 0$ with equality if and only if $p = p_\alpha^\star$, the maximum of $J$ is attained uniquely at $p = p_\alpha^\star$.
\end{proof}

\subsection{Proof of Proposition~\ref{lem:exact-corrected-conditional}}
\label{app:proof-exact-corrected}

\exactcorrectedlemma*

\noindent We prove a more general result from which the lemma follows.
For any state \(s=(A,\bx_A)\) and uncommitted position \(i\in R(s)\), let
\(R\coloneqq R(s)\setminus\{i\}\) and define
\[
  m_i(v\mid s) \;\coloneqq\; q(x_i=v\mid s),
  \qquad
  w_i(v\mid s) \;\coloneqq\; \sum_{\bx_R\in\mathcal{V}^R} q(\bx_R \mid s, x_i=v)^\alpha.
\]
Then
\[
  p_\alpha^\star(x_i=v\mid s)
  \;=\;
  \frac{m_i(v\mid s)^\alpha\, w_i(v\mid s)}
       {\sum_{u\in\mathcal{V}} m_i(u\mid s)^\alpha\, w_i(u\mid s)}.
\]
The logit form in \Cref{lem:exact-corrected-conditional} follows by noting
\(m_i(v\mid s)^\alpha \propto \exp(\alpha\,\boldsymbol{\ell}_{i,v}(s))\) and
\({\color{argmaxcolor}\delta_{i,v}(s)}=\log w_i(v\mid s)\).

\begin{proof}
By definition of conditional probability under \(p_\alpha^\star\),
\[
p_\alpha^\star(x_i=v\mid s)
=
\frac{\sum_{\bx_R\in V^R} p_\alpha^\star(x_i=v,\bx_R\mid s)}{\sum_{u\in V}\sum_{\bx_R\in V^R} p_\alpha^\star(x_i=u,\bx_R\mid s)}.
\]
Using \(p_\alpha^\star(\cdot\mid s)\propto q(\cdot\mid s)^\alpha\) and canceling the common normalizer,
\[
p_\alpha^\star(x_i=v\mid s)
=
\frac{\sum_{\bx_R} q(x_i=v,\bx_R\mid s)^\alpha}{\sum_{u}\sum_{\bx_R} q(x_i=u,\bx_R\mid s)^\alpha}.
\]
Apply the chain rule under \(q(\cdot\mid s)\):
\[
q(x_i=v,\bx_R\mid s)
=
q(x_i=v\mid s)\, q(\bx_R\mid s,x_i=v)
=
m_i(v\mid s)\, q(\bx_R\mid s,x_i=v).
\]
Raising to the power \(\alpha\) and summing over \(\bx_R\) gives
\[
\sum_{\bx_R} q(x_i=v,\bx_R\mid s)^\alpha
=
m_i(v\mid s)^\alpha \sum_{\bx_R} q(\bx_R\mid s,x_i=v)^\alpha
=
m_i(v\mid s)^\alpha\, w(v\mid s),
\]
which yields the first displayed formula after normalization over \(v\in V\).

For the logit form, note that
\(m_i(v\mid s)=\exp(\boldsymbol{\ell}_v(s))/\sum_{u\in V}\exp(\boldsymbol{\ell}_u(s))\), hence
\(m_i(v\mid s)^\alpha \propto \exp(\alpha \boldsymbol{\ell}_v(s))\), where the proportionality constant does not depend on \(v\).
Therefore \(m_i(v\mid s)^\alpha w(v\mid s)\propto \exp(\alpha \boldsymbol{\ell}_v(s)+{\color{argmaxcolor}\delta_v(s)})\), and normalizing over \(v\) gives
the corrected-logit softmax expression.
\end{proof}

\end{document}